\documentclass[english,a4paper,12pt]{article}
\PassOptionsToPackage{pdfpagelabels=false}{hyperref}

\usepackage{amsmath, amsxtra, amsfonts, amssymb, amstext}
\usepackage{amsthm}
\usepackage{booktabs}
\usepackage{fullpage}
\usepackage{nicefrac}
\usepackage{xspace}
\usepackage[noadjust]{cite}
\usepackage{url}
\usepackage{graphics}
\usepackage[usenames,dvipsnames]{xcolor}
\usepackage[colorlinks]{hyperref}
\definecolor{linkblue}{rgb}{0.1,0.1,0.8}
\hypersetup{colorlinks=true,linkcolor=linkblue,filecolor=linkblue,urlcolor=linkblue,citecolor=linkblue}
\usepackage[algo2e,ruled,vlined,linesnumbered]{algorithm2e}
\usepackage{wrapfig}

\newtheorem{theorem}{Theorem}
\newtheorem{lemma}[theorem]{Lemma}
\newtheorem{proposition}[theorem]{Proposition}

\newcommand{\R}{\mathbb{R}}
\newcommand{\Z}{\mathbb{Z}}

\renewcommand{\epsilon}{\varepsilon}

\DeclareMathOperator{\eq}{eq}
\DeclareMathOperator{\mut}{mut}
\DeclareMathOperator{\cross}{cross}

\newcommand{\Bin}{\mathcal{B}}

\newcommand{\assign}{\leftarrow}

\newcommand{\ga}{$(1 + (\lambda,\lambda))$~GA\xspace}

\newcommand{\onemax}{\textsc{OneMax}\xspace}
\newcommand{\OM}{\textsc{Om}\xspace}

\begin{document}
{\sloppy

\title{Optimal Parameter Settings for the $(1+(\lambda, \lambda))$ Genetic Algorithm\thanks{A short version of this paper with many proofs omitted appeared at GECCO'16.} }

\author{
Benjamin Doerr, \'Ecole Polytechnique, France}
\maketitle
 
\begin{abstract}
The $(1+(\lambda,\lambda))$ genetic algorithm is one of the few algorithms for which a super-constant speed-up through the use of crossover could be proven. So far, this algorithm has been used with parameters based also on intuitive considerations. In this work, we rigorously regard the whole parameter space and show that the asymptotic time complexity proven by Doerr and Doerr (GECCO 2015) for the intuitive choice is best possible among all settings for population size, mutation probability, and crossover bias. 
\end{abstract}

\sloppy{
\section{Introduction}
The $(1+(\lambda,\lambda))$ genetic algorithm (\ga) was first proposed in~\cite{DoerrDE13} (see~\cite{DoerrDE15} for the journal version). It is a simple evolutionary algorithm that uses a biased crossover with a parent individual in a way that can be interpreted as a repair mechanism. It was the first (unbiased in the sense of Lehre and Witt~\cite{LehreW12}) evolutionary algorithm to provably optimize any \onemax test function in time asymptotically smaller than the famous $\Theta(n \log n)$ barrier~\cite{DoerrD15tight}, but showed a favorable performance in experiments also for several other classic test functions~\cite{DoerrDE15} and combinatorial optimization problems~\cite{GoldmanP14,MironovichB15}. This algorithm (together with, e.g.,~\cite{JansenW02,FischerW04,Sudholt05,DoerrHK12}) also is one of the still surprisingly few examples where crossover could be rigorously proven to useful.

One difficulty when using the \ga is that it comes with several parameters, namely an offspring population size $\lambda$, a mutation probability $p$, and a crossover bias $c$. In all previous works, these parameters were chosen by combining rigorous and intuitive arguments (see Section~\ref{sec:parameters}). While the results, e.g., an $O(n \sqrt{\log n})$ runtime for all \onemax functions in the first paper~\cite{DoerrDE13}, indicate that these intuitive choices were not too bad, all existing work leaves open the possibility that completely different parameter choices give an even better performance.

For this reason, in this work we rigorously prove a lower bound valid for the whole $3$-dimensional parameter space. Our lower bound coincides with the runtime proven in~\cite{DoerrD15tight} for the intuitive choices taken there. Consequently, these parameter choices were optimal. As a side product of this result, we also see that not many other parameter choices can lead to this optimal runtime. We have to defer the precise statement of our results (Theorem~\ref{thm:lower}) to a point where the algorithms and its parameters have been made precise. 

From a broader perspective, our results and in particular the partial results that lead to it, give a clearer picture on how to choose the parameters in the \ga, also for optimization problems beyond the \onemax test function class (see the Conclusion section). 

From the methodological standpoint, this is one of very few theoretical works that analyze evolutionary algorithms involving more than one parameter. We observe that the parameters do not have an independent influence on the runtime, but that they interact in a difficult to foresee manner. A similar observation was made in~\cite{GiessenW15}, who proved for the $(1+\lambda)$ EA that the mutation probability has a decisive influence on the performance when the population size $\lambda$ is asymptotically smaller than the cut-off point $\ln(n) \ln\ln(n) / \ln\ln\ln(n)$, where as it have almost no influence when $\lambda = \omega(\ln(n) \ln\ln(n) / \ln\ln\ln(n))$. Such non-separable parameter influences, naturally, makes the analysis of a multi-dimensional parameter space more difficult. A second difficulty we had to overcome is that, while only few parameter configuration yields the asymptotically optimal runtime, a quite large set of combinations including some that are far from the optimal ones still lead to a runtime very close to the optimal one (see the remark at the end of Section~\ref{sec:overview}). While this is good from the application point of view (missing the absolutely optimal parameters is less harmful), from the viewpoint of proving our results it means that there is not much room for non-sharp estimates. Overcoming these difficulties, we are also optimistic that this work helps future work in the analysis of multi-dimensional parameter spaces.

\section{The \texorpdfstring{\ga}{(1+(lambda,lambda))~GA}}
\label{sec:algorithm}\label{sec:parameters}

The \ga is a fairly simple evolutionary algorithm using crossover. It was introduced in~\cite{DoerrDE13,DoerrDE15}, some experimental results can be found in~\cite{GoldmanP14}. Its pseudo-code is given in Algorithm~\ref{alg:GA}. 

\begin{algorithm2e}[t]%
	\textbf{Initialization:} 
	Choose $x \in \{0,1\}^n$ uniformly at random and evaluate $f(x)$\;
 \textbf{Optimization:}
\For{$t=1,2,3,\ldots$}{
\textbf{Mutation phase:}
Sample $\ell$ from $\Bin(n,p)$\label{line:L}\;
\For{$i=1, \ldots, \lambda$\label{line:mutstart}}{
$x^{(i)} \assign \mut_{\ell}(x)$ and evaluate $f(x^{(i)})$\;
}
Choose $x' \in \{x^{(1)}, \ldots, x^{(\lambda)}\}$ with $f(x')=\max\{f(x^{(1)}), \ldots, f(x^{(\lambda)})\}$ u.a.r.\label{line:mutend}\;
\textbf{Crossover phase:}
\For{$i=1, \ldots, \lambda$\label{line:costart}}{
$y^{(i)} \assign \cross_{c}(x,x')$ and evaluate $f(y^{(i)})$\label{line:co}\; 
}
Choose $y \in \{y^{(1)}, \ldots, y^{(\lambda)}\}$ with $f(y)=\max\{f(y^{(1)}), \ldots, f(y^{(\lambda)})\}$ u.a.r.\label{line:coend}\;
\textbf{Selection step:}
\lIf{$f(y)\geq f(x)$}{$x \assign y$\;
}
}
\caption{The \ga, maximizing a given function $f : \{0,1\}^n \to \R$, with offspring population size $\lambda$, mutation probability $p$, and crossover bias $c$. The mutation operator $\mut_\ell$ generates an offspring from one parent by flipping exactly $\ell$ random bits (without replacement). The crossover operator $\cross_c$ performs a biased uniform crossover, taking bits independently with probability $c$ from the second argument and with probability $1-c$ from the first parent.}
\label{alg:GA}
\end{algorithm2e}

The \ga is initialized with a solution candidate drawn uniformly at random from $\{0,1\}^n$. It then proceeds in iterations consisting of a mutation, a crossover, and a selection phase. In an important contrast to many other genetic algorithms, the mutation phase \emph{precedes} the crossover phase. This allows to use crossover as a repair mechanism, as we shall discuss in more detail below.

In the \emph{mutation phase} of the \ga, we create $\lambda$ offspring from the current-best solution $x$ by applying to it the mutation operator $\mut_{\ell}(\cdot)$, which flip $\ell$ positions uniformly at random. In other words, $\mut_\ell(x)$ is a bit-string in which for $\ell$ random positions $i$ the entry $x_{i} \in \{0,1\}$ is replaced by $1-x_i$. The \emph{step size} $\ell$ is chosen randomly according to a binomial distribution $\Bin(n,p)$ with $n$ trials and success probability $p$. To ensure that all mutants have the same distance from the parent $x$, and thus to not bias the selection by different distances from the parent, the same $\ell$ is used for all $\lambda$ offspring. The fitness of the $\lambda$ offspring is computed and the best one of them,  $x'$, is selected to take part in the crossover phase. If there are several offspring having maximal fitness, we pick one of them uniformly at random (u.a.r.).

When $x$ is already close to an optimal solution, the offspring created in the mutation phase are typically all of much worse fitness than $x$. Our hope is though that they have discovered some parts of the optimum solution that is not yet reflected in $x$. In order to preserve these parts while at the same time not destroying the good parts of $x$, the \ga creates in the \emph{crossover phase} $\lambda$ offspring from $x$ and $x'$. Each one of these offspring is sampled from a uniform crossover with bias $c$ to take an entry from $x'$, that is, each offspring $y^{(i)} := \cross_c(x,x')$ is created by independently for each position $j$ setting $y^{(i)}_j := x'_i$ with probability $c$ and taking $y^{(i)}:=x_j$ otherwise. 
Again we evaluate the fitness of the $\lambda$ crossover offspring and select the best one of them, which we denote by $y$. If there are several offspring of maximal fitness, we simply take one of them uniformly at random.\footnote{In~\cite[Section 4.4]{DoerrDE13} and~\cite{DoerrDE15} a slightly different selection rule is suggested for the crossover phase, which is more suitable for functions with large plateaus of the same fitness value. 
Since we consider in this work only the \onemax function, for which both algorithms are identical by symmetry reasons, we refrain from stating in Algorithm~\ref{alg:GA} the slightly more complicated version proposed there, which selects the parent solution $x$ only if there is no offspring $\neq x$ of fitness value at least as good as the one of $x$.}

Finally, in the \emph{selection step} the previous-best solution $x$ is replaced by new $y$ if and only if the fitness of $y$ is at least as good as the one of $x$.

As common in the runtime analysis community, we do not specify a termination criterion. The simple reason is that we study as a theoretical performance measure the expected number of function evaluations that the \ga performs until it evaluates for the first time a search point of maximal fitness (the so-called optimization time). Of course, for an application to a real problem a termination criterion has to be specified.

\subsection*{Parameter Choices}

The \ga comes with a set of parameters, namely the mutation probability $p$, the crossover bias $c$, and the off-spring population size $\lambda$. If $\ell \sim \Bin(n,p)$, then observations that $\cross_c(x,\mut_\ell(x))$ has the distribution of an individual created from $x$ via standard bit mutation with mutation rate $pc$. Since $1/n$ is an often preferred choice for the mutation rate, the authors of~\cite{DoerrDE13} suggest to choose $p$ and $c$ in a way that $pc = 1/n$. Note that due to the two intermediate selection steps, the final offspring $y$ has a very different distribution than standard bit mutation with rate $pc$ -- otherwise the \ga could not obtain runtimes better than $\Theta(n \log n)$.

Parameterizing $p = k/n$, that is, $k$ denotes the average number of bits flipped by an application of the mutation operator, the above suggestion is to take $c = 1/k$. For these settings, a runtime analysis for the \onemax test function in~\cite{DoerrDE13} gave an upper bound for the runtime of $O((\frac 1k + \frac 1\lambda) n \log n + (k + \lambda) n)$. From this an some experiments, the suggestion to take $k = \lambda$ was derived, reducing the parameter space to the single parameter $\lambda$. Since only an upper bound for the runtime was used to obtain this suggestion, again this is an intuitive argument, but not a rigorous one. 

For the parameter settings $p = \lambda/n$, $c = 1/\lambda$, and arbitrary $\lambda$ a more precise runtime analysis~\cite{DoerrD15tight}, again on the \onemax test function class, gave a tight order of magnitude for the expected runtime of \[\Theta\left(\max\left\{\frac{n \log(n)}{\lambda}, \frac{n \lambda \log\log(\lambda)}{\log(\lambda)}\right\}\right),\]
which is minimized exactly by the parameter choice 
$\lambda = \Theta(\sqrt{\log(n) \log\log(n) / \log\log\log(n)})$. As said above, we shall prove that also all other choice of mutation probability, crossover bias, and offspring population size lead to this or a worse runtime.

\section{Runtime Analysis}

\emph{Runtime analysis} is one of the most successful theoretical tools to understand the performance of evolutionary algorithms. The \emph{runtime} or \emph{optimization time} of an algorithm (e.g., our \ga) on a problem instance (e.g., the \onemax function) is the number of fitness evaluations that are performed until for the first time an optimal solution is evaluated. 

If the algorithm is randomized (like our \ga), this is a random variable $T$, and we usually make statements on the expected value $E[T]$ or give bounds that hold with some high probability, e.g., $1 - 1/n$. When regarding a problem with more than one instance (e.g., traveling salesman instance on $n$ cities), we take a worst-case view. This is, we regard the maximum expected runtime over all instances, or we make statements like that the runtime satisfies a certain bound for all instances. 

In this work, the optimization problem we regard is the classic \onemax test problem consisting of the single instance $\OM: \{0,1\}^n \to \{0,1,\ldots,n\}; x \mapsto \sum_{i = 1}^n x_i$, that is, maximizing the number of ones in a bit-string. 
Despite the simplicity of the \onemax problem, analyzing randomized search heuristics on this function has spurred much of the progress in the theory of evolutionary computation in the last 20 years, as is documented, e.g., in the recent textbook~\cite{Jansen13}. 

Of course, when regarding the performance on a single test instance, then we should ensure that the algorithm does not exploit the fact that there is only one instance. A counter-example would be the algorithm that simply evaluates and outputs $x^* = (1,\dots,1)$, giving a perfect runtime of $1$. One way of ensuring this is that we restrict ourselves to unbiased algorithms (see~\cite{LehreW12}) which treat bit-positions and bit-values in a symmetric fashion. Consequently, an unbiased algorithm for the \onemax problem has the same performance on all problems with isomorphic fitness landscape, in particular, on all (generalized) \onemax functions $\OM_z : \{0,1\}^n \to \{0,1,\ldots,n\}; x \mapsto \eq(x,z)$ for $z \in \{0,1\}^n$, where $\eq(x,z)$ denotes the number of bit-positions in which $x$ and $z$ agree. It is easy to see that the \ga is unbiased (for all parameter settings).

\section{Notation and Technical Tools}

In this section, besides fixing some very elementary notation, we collect the main technical tools we shall use. Mostly, these are large deviations bounds of various types. For the convenience of the reader, we first state the known ones. We then prove a tail bound for sums of geometric random variables with expectations bounded from above by the reciprocals of the first positive integers. We finally state the well-known additive drift theorem.

\subsection{Notation}

We write $[a..b]$ to denote the set $\{z \in \Z \mid a \le z \le b\}$ of integers between $a$ and $b$. 
We write $\log(n)$ to denote the binary logarithm of $n$ and $\ln(n)$ to denote the natural logarithm of $n$. However, to avoid unnecessary case distinctions when taking iterated logarithms, we define $\log(n):=1$ for all $n \le 2$ and $\ln(n) := 1$ for all $n \le e$.
For the readers' convenience, we now collect some tools from probability theory which we will use regularly. 

We occasionally need the expected value of a binomially distributed random variable $X \sim \Bin(n,p)$ conditional on that the variable has at least a certain value $k$. An intuitive (but wrong) solution to this question is that this $E[X | X \ge k]$ should be around $k+p(n-k)$, because we know already that at least $k$ of the $n$ independent trials are successes and the remaining $(n-k)$ trials still have their independent success probability of $p$. While this argument is wrong, an upper bound of this type can be shown by elementary means. Since we have not seen this made explicit in the EA literature, we shall also give the short proof.

\begin{lemma}\label{lem:condbinomial}
  Let $X$ be a random variable with binomial distribution with parameters $n$ and $p \in [0,1]$. Let $k \in [0..n]$. Then \[E[X \mid X \ge k] \le k + (n-k)p \le k + E[X].\]
\end{lemma}

\begin{proof}
  Let $X_1, \dots, X_n$ be independent binary random variables with $\Pr[X_i = 1] = p$ for all $i \in [1..n]$. Then $X = \sum_{i=1}^n X_i$ has a binomial distribution with parameters $n$ and $p$. Conditioning on $X \ge k$, let $\ell := \min\{i \in [1..n] \mid \sum_{j=1}^i X_j = k\}$. Then $E[X \mid X \ge k] = \sum_{i = 1}^n \Pr[\ell = i \mid X \ge k ] E[X \mid \ell = i]$. Note that $\ell \ge k$ by definition. Note also that $(X \mid \ell = i) = k + \sum_{j = i+1}^n X_j$ with unconditioned $X_j$. In particular, $E[X \mid \ell = i] = k + (n-i)p$. Consequently, $E[X \mid X \ge k] = \sum_{i = 1}^n \Pr[\ell = i \mid X \ge k] E[X \mid \ell = i] \le \sum_{i = k+1}^n \Pr[\ell = i \mid X \ge k ] (k + (n-k)p) = k + (n-k)p$.
\end{proof}

Also, we shall use the following well-known fact.

\begin{lemma}\label{lem:integerexpectation}
  Let $X$ be a non-negative integral random variable. Then $E[X] = \sum_{i = 1}^\infty \Pr[X \ge i]$.
\end{lemma}

\subsection{Known Chernoff Bounds}

The following \emph{large deviation bounds} are well-known and can be found, e.g., in~\cite{Doerr11bookchapter}. We call all these bounds Chernoff bounds despite the fact that it is now known that some have been found earlier by other researchers.

\begin{theorem}[Classic Chernoff bounds]\label{thm:chernoff} 
  Let $X_1, \ldots, X_n$ be independent random variables taking values in $[0,1]$. Let $X = \sum_{i = 1}^n X_i$. 
  \begin{enumerate}
    \item \label{multchernoffupperstrong} Let $\delta \ge 0$. Then $\Pr[X \ge (1+\delta) E[X]] \le (\frac{e^\delta}{(1+\delta)^{1+\delta}})^{E[X]}$. 
    \item \label{multchernoffupper} Let $\delta \in [0,1]$. Then $\Pr[X \ge (1+\delta) E[X]] \le \exp(-\delta^2 E[X]/3)$. 
    \item \label{chernoffzweihoch} Let $d \ge 6 E[X]$. Then $\Pr[X \ge d] \le 2^{-d}$.
    \item \label{multchernofflower} Let $\delta \in [0,1]$. Then $\Pr[X \le (1-\delta) E[X]] \le \exp(-\delta^2 E[X]/2)$.
    \item \label{addchernoff} Let $X_1, \dots, X_n$ be independent random variables each taking values in some interval of length at most one. Let $X = \sum_{i=1}^n X_i$. Let $\lambda \ge 0$. Then $\Pr[X \le E[X]-\lambda] \le \exp(-2 \lambda^2 / n)$ and $\Pr[X \ge E[X]+\lambda] \le \exp(-2 \lambda^2 / n)$.
  \end{enumerate}
\end{theorem}

Chernoff bounds also hold for \emph{hypergeometric} distributions. 
Let $A$ be any set of $n$ elements. Let $B$ be a subset of $A$ having $m$ elements. If $Y$ is a random subset of $A$ of $N$ elements (chosen uniformly at random from all $N$-element subsets of $A$, then $X :=|Y \cap B|$ has a hypergeometric distribution with parameters $(n,N,m)$. 

\begin{theorem}[Chernoff bounds for hypergeometric distributions]\label{thm:hypergeomchernoff}
  If $X$ has a hypergeometric distribution with parameters $(n,N,m)$, then $E[X]=Nm/n$ and $X$ satisfies all Chernoff bounds given in Theorem~\ref{thm:chernoff}.
\end{theorem}

\subsection{Drift Analysis}

\emph{Drift analysis} comprises a couple of methods to derive from information about the expected progress (e.g., in terms of the fitness distance) a result about the time needed to achieve a goal (e.g., finding an optimal solution). We shall several times use the following \emph{additive drift} theorem from~\cite{HeY01} (see also Theorem~2.7 in~\cite{OlivetoY11bookchapter}).

\begin{theorem}[additive drift theorem]\label{thm:drift}
  Let $X_0, X_1, ...$ be a sequence of random variables taking values in a finite set $S \subseteq \R_{\ge 0}$. Let $T := \min\{t \ge 0 \mid X_t=0\}$. Let $\delta > 0$. 
  \begin{enumerate}
  \item[(i)] If for all $t$, we have $E[X_t - X_{t+1} | X_t>0] \ge \delta$, then $E[T | X_0] \le X_0 / \delta$. 
  \item[(ii)] If for all $t$, we have $E[X_t - X_{t+1} | X_t>0] \le \delta$, then $E[T | X_0] \ge X_0 / \delta$. 
  \end{enumerate}
\end{theorem}

In many situation, the progress $X_t - X_{t+1}$ is stronger when the process is far from the target, that is, when $X_t$ is large. A particular, but seemingly very common special case is that the progress is indeed proportional to $X_t$. Such a situation is called \emph{multiplicative drift}. Drift theorems giving upper bounds for the hitting time were given in~\cite{DoerrJW12} and~\cite{DoerrG13algo}. Transforming upper bounds on a multiplicative progress into good lower bounds for hitting times requires additional assumptions. Witt gives the following very useful theorem (Theorem 2.2 in~\cite{Witt13j}).

\begin{theorem}[multiplicative drift, lower bound]\label{thm:multidriftlower}
Let $S \subset \R$ be a finite set of positive numbers with minimum $1$. Let $X_0, X_1, \dots$ be a sequence of random variables over $S$ such that $X_t \ge X_{t+1}$ for any $t \ge 0$. Let $s_{\min} > 0$. Let $T$ be the random variable that gives the
first point in time $t \ge 0$ for which $X_t \le s_{\min}$. If there exist positive reals $\beta, \delta \le 1$ such that,
for all $s > s_{\min}$ and all $t \ge 0$ with $\Pr[X_t = s] > 0$,
\begin{itemize}
	\item[(1)] $E[X_t - X_{t+1} | X_t = s] \le \delta s$,
	\item[(2)] $\Pr[X_t - X_{t+1} \ge \beta s | X_t = s] \le \beta \delta / \ln(s)$,
\end{itemize}
then for all $S_0 \in S$ with $\Pr[X_0  = s_0] > 0$, we have $E[T | X_0 = s_0] \ge \frac{\ln(s_0) - \ln(s_{\min})}{\delta} \cdot \frac{1-\beta}{1+\beta}$.
\end{theorem}

\section{Main Result and Proof}

As described in Section~\ref{sec:parameters}, a combination of intuitive considerations and rigorous work made \cite{DoerrDE13,DoerrD15tight} suggest the parameter choice $\lambda = \lambda^* := \sqrt{\frac{\log(n) \log\log(n)}{\log\log\log(n)}}$, $p^* = \lambda^*/n$, and $c^* = 1 / \lambda^*$ for the optimization of the \onemax test function class, yielding an expected optimization time of $F^* = \frac{n \log n}{\lambda^*} = n \sqrt{\frac{\log(n) \log\log\log(n)}{\log\log(n)}}$. It was also proven that with $p$ and $c$ functionally depending on $\lambda$ as above, $\lambda = \Theta(\lambda^*)$ is the optimal choice and the only optimal choice. 

In this section, we complete this picture by proving rigorously that no combination of the parameters $p$, $c$, and $\lambda$, all possibly depending on $n$, can lead to an expected optimization time of asymptotic order strictly better than $F^*$. We also show that not many  parameter combinations can give this optimal expected runtime.

\begin{theorem}\label{thm:lower}
  Let $\lambda^* := \sqrt{\frac{\log(n) \log\log(n)}{\log\log\log(n)}}$ and $F^* = \frac{n \log n}{\lambda^*} = n \sqrt{\frac{\log(n) \log\log\log(n)}{\log\log(n)}}$. 
  \begin{itemize}
  \item For arbitrary parameters $\lambda \in [0..n]$, $p \in [0,1]$ and $c \in [0,1]$, all being functions on $n$, the \ga has an expected optimization time of $E[F] = \Omega(F^*)$. 
  \item If some parameter combination $(\lambda,p,c)$ leads to an expected optimization time of $E[F] = \Theta(F^*)$, then 
  \begin{itemize}
  \item $\lambda = \Theta(\lambda^*)$, 
  \item $p = \Omega(\lambda^*/n)$ and $p = (1/n)\exp(O(\sqrt{{\log(n) \log\log\log(n)} / {\log\log(n)}}\,))$, and 
  \item $c = \Theta(1/pn)$.  
  \end{itemize}
  \end{itemize}
\end{theorem}

We remark that \emph{the same lower bound holds for the natural modification of the \ga in which the best of all mutation and crossover offspring competes in the final selection step with the parent individual} (and not only the best crossover offspring). The proofs below are written up in a way that this is easy to check, but to keep the paper readable we do not explicitly formulate all statements for both version of the algorithm. Consequently, for the \onemax testfunction, this modification does not give an asymptotic runtime improvement. In a practical application, however, there is no reason to not exploit possible exceptionally good mutation offspring. So here this modification seems very advisable.

To ease the presentation, we shall always parameterize these values by $p = k/n$ and $c = r/k$ for some $k \in [0,n]$ and $r \in [0,k]$ (hence $k$ and $r$ may also depend on $n$). In this language, the previously suggested values are $k^* = \lambda^*$ and $r^* = 1$, and the main result of this work is that 
\begin{enumerate}
\item[(i)] no parameter setting gives a better expected optimization time than the $\Theta(F^*)$ stemming from these parameters, and 
\item[(ii)] any parameter tuple $(\lambda,k,r)$ that leads to an asymptotic optimization time of $\Theta(F^*)$ satisfies $\lambda = \Theta(\lambda^*)$, $k = \Omega(k^*)$ and $k = \exp(O(\sqrt{{\log(n) \log\log\log(n)} / {\log\log(n)}}\,))$, and $r = \Theta(r^*)$.
\end{enumerate}

A side remark: Another implicit parameter choice done in~\cite{DoerrDE13} is to use the same offspring population size $\lambda$ for the mutation phase and the crossover phase. One could well imagine having different numbers $\lambda_m$ and $\lambda_c$ of offspring for both phases. This may make sense in practical applications or when performing a theoretical analysis that takes care of constant factors. In this work, where we are only precise up to the asymptotic order of magnitude, the optimization time is of asymptotic order equal to the product of the number of iterations and $\max\{\lambda_m,\lambda_c\}$. Hence, unless one believes that a smaller offspring population size can reduce the number of iterations (which is not what our proofs suggest), there is for us no use of not taking both offspring population sizes equal to $\max\{\lambda_m,\lambda_c\}$.

\subsection{Overview of the Proof}\label{sec:overview}

Given apparent difficulty (see~\cite{DoerrD15tight}) of determining the runtime of the \ga already for settings $k = \lambda$ and $r=1$ suggested in~\cite{DoerrDE13}, the common approach of determining the optimal parameter settings by conducting a precise runtime analysis for all parameter combinations $(\lambda,k,r)$ seems not very promising. Therefore, shall rather analyze particular parts of  the optimization process in detail and from these extract necessary conditions for the parameters to allow an expected optimization time of order $O(F^*)$. To make it more visible how the different arguments work together, let us start with a brief overview of the analysis.

Let a tuple $(\lambda, p = k/n, c = r/k)$ as described above be given. We denote by $T$ the number of iterations the \ga with these parameters performs until an optimal solution is found (we have $T=0$ if the random initial search point is already optimal). We denote by $F$ the optimization time of this \ga, that is, the number of fitness evaluations performed until an optimal solution is evaluated. This is one if the random initial search point was optimal. We roughly have $F \approx 2 \lambda T$, but see Proposition~\ref{prop:lbtf} and the text around it for the details. 

We say that a tuple of parameters is \emph{optimal} if the resulting optimization time is $O(F^*)$. This is, for the moment, a slight abuse of language, but as this section will show, these are indeed the parameters that lead to the asymptotically optimal runtime, since (as we will see) no better runtime than $\Omega(F^*)$ can be achieved with any parameter setting. The proof of the Theorem~\ref{thm:lower} then consists of the following arguments, which all can be shown independent of the others. Since we aim at an asymptotic result only, we can freely assume that $n$ is sufficiently large.

\begin{itemize}
\item In Lemma~\ref{lem:lblambdalarge}, we make the elementary observation that $E[F] \ge \min\{\lambda,2^{n}\}/2$. Consequently, $\lambda \le 2 F^*$ in any optimal parameter set.

\item In Lemma~\ref{lem:lblargek}, we show that \[E[F] = \min\{\Omega(r^{-1} \exp(\Theta(r)) n \log n), \exp(\Omega(r)) n^2 \log n, \exp(\Omega(n^{1/16}))\}\] when $k \ge \sqrt n$ and $\lambda = \exp(o(n^{1/16}))$.  Since this runtime is at least $\Omega(n \log n)$, together with the previous item (showing that $\lambda$ cannot be too large), we obtain than $k \le \sqrt n$ in an optimal parameter set. 

\item In Lemma~\ref{lem:lbcoupon}, we show that for $0 < k \le n/12$, we have $E[F] = \Omega(\frac{n \log n}{k})$. Hence $k = \Omega(\lambda^*)$ in an optimal parameter setting.

\item In Lemma~\ref{lem:lbsmallk}, we show that when $\omega(1) = k \le \sqrt n$, then $E[F] =  \Omega(n \log n \min\{\tfrac{\exp(\Omega(r))}{\lambda r}, \tfrac{n^3}{\lambda}, \tfrac{\exp(\Omega(k))}{k}\})$. Since we know already that $\lambda \le n^3$ and $k = \omega(1)$ in an optimal parameter setting, this result  implies that an optimal parameter set has $\lambda = \Omega(\lambda^* \exp(\Omega(r))/r)$. 

\item In Lemma~\ref{lem:lblambdak}, we show $E[F] = \Omega(n \lambda / k)$ when $k \le n/4$ (which we know already). Consequently, in an optimal set of parameters $\lambda$ cannot be excessively large, e.g., $\lambda \le \exp(k/120)$.

\item In Lemma~\ref{lem:lblambda}, we show that if $k \le n/80$, $\lambda \le \exp(k/120)$, $\lambda = \exp(o(n))$, and $\lambda = \omega(1)$---all of this holds in an optimal parameter setting as shown above---then $E[F] = \Omega(\frac{n \lambda \log\log(\lambda)}{r \log \lambda})$. This result together with Lemma~\ref{lem:lbsmallk} implies that the optimal runtime is $\Theta(F^*)$ and that we have $\lambda = \Theta(\lambda^*)$ and $r = \Theta(1)$ in an optimal parameter setting.

\end{itemize}

This shows the main claim of this work, namely that $F^*$ is asymptotically the best runtime one can achieve with a clever choice of all parameters of the \ga. The above also shows that an optimal parameter set has $\lambda = \Theta(\lambda^*)$ and $r = \Theta(1)$. For the mutation probability, the above only yields $k = \Omega(\lambda^*)$ and $k = O(\sqrt n)$. In Lemma~\ref{lem:kvalue}, we show that $k = \exp(O(\sqrt{{\log(n) \log\log\log(n)} / {\log\log(n)}}\,))$ is a necessary condition for having a $\Theta(F^*)$ runtime.

We do not know if the interval of optimal $k$ values can be further reduced. An inspection of the upper bound proof in~\cite{DoerrD15tight} suggests that, with more effort than there, also slightly larger $k$-values than $\Theta(\lambda^*)$ (together with $\lambda = \Theta(\lambda^*)$ and $r = \Theta(1)$) could lead to the optimal expected runtime of $\Theta(F^*)$. We do not follow up on this question, because we do not feel that it justifies the effort of extending the technical proof of~\cite{DoerrD15tight}. It is quite clear that there is no algorithmic advantage of using a larger than necessary $k$-value. The main (unfavorable) difference would be that than an efficient implementation of the mutation operator in expected time $\Theta(k)$ would have an increased complexity.

We face two main difficulties in this proof. One are the apparent dependencies introduced by the two intermediate selection steps and the fact that all mutation offspring have the same Hamming distance from the parent. That the latter creates additional challenges can be easily seen in the lengthy proof of Lemma~\ref{lem:lbcoupon}, which simply tries to use the classic argument that one needs at least a total number of $\Theta(n \log n)$ bit-flips to make sure that each initially incorrect bit was flipped at least once. 

The second difficulty is that even parameter combinations that are far from those leading to the optimal runtime can lead to runtimes very close to the optimal one. An example (given here without proof) is that for say $k=\sqrt n$ and $\lambda = \lambda^*$ and $r = 1$, the optimization process strongly resembles the one of the $(1+\lambda)$ EA with $\lambda$ below the cut-off point. Consequently, the \ga for these parameters has an optimization time of $\Theta(n \log n)$, which is relatively close to $F^*$ given uncommonly large mutation probability. 

\subsection{Proofs}

This this longer subsection, we prove the results outlined above. We frequently use the following notation. For  $x \in \{0,1\}^n$, we call $d(x) := n - \OM(x)$ its \emph{fitness distance}. Let $x, x', y \in \{0,1\}^n$. Then \[g(x,x') := |\{ i \in [1..n] \mid x_i = 0 \wedge x'_i=1\}|\] is the number of  \emph{good bits of $x'$ (with respect to $x$)}. Analogously, \[b(x,x') := |\{ i \in [1..n] \mid x_i = 1 \wedge x'_i=0\}|\] is the number of  \emph{bad bits of $x'$ (with respect to $x$)}. Note that, trivially, $g(x,x') + b(x,x') = H(x,x')$, the Hamming distance of $x$ and $x'$. Similarly, we define ``the number of good bits of $x'$ that made it into $y$'' and ``the number of bad bits of $x'$ that made it into $y$'' by
\begin{align*}
  g(x,x',y) &:= |\{i \in [1..n] \mid x_i = 0 \wedge x'_i = 1 \wedge y_i = 1\}|,\\
  b(x,x',y) &:= |\{i \in [1..n] \mid x_i = 1 \wedge x'_i = 0 \wedge y_i = 0\}|.
\end{align*}

In the following, we always assume that we consider a run of the \ga with general parameter setting $\lambda$, $p = k/n$, and $c = r/k$, which may all depend on the problem size $n$. Since we are interested in an asymptotic result, we may assume that $n$ is sufficiently large. We use the variables of the algorithm description, e.g., $x$, $x^{(i)}$, $x'$, etc. without further explicit reference to the algorithm.

We now prove the ingredients forming the proof of the main result. We prove these results not only for the minimal parameter range needed in the proof of the main result, but rather for those ranges where the main arguments work well. At the same time, we do not aim at the absolutely widest parameter range and we occasionally do not aim at the sharpest possible bound if this would significantly increase the proof complexity. We aim at keeping the proofs of the partial results independent, both to ease reading and to allow an easier understanding how the main proof decomposes into the partial results. For this reason, all of the following lemmas are proven independently apart from possibly relying on the two elementary propositions~\ref{prop:lbsmallsteps} and~\ref{prop:lbtf}.

The first of these proposition is a technical tool showing that extraordinarily large fitness gains occurs rarely. This allows in the following to assume that the algorithm indeed once has a parent individual $x$ with roughly a certain fitness.

\begin{proposition}\label{prop:lbsmallsteps}
	Let $x$ be a search point with $d := d(x)$ satisfying $d \le 0.6n$. Then the probability that one iteration of the \ga with arbitrary parameter settings creates a search point $y$ with $d(y) \le d/2$, is $\lambda(\lambda+1) \exp(-\Omega(d))$.
\end{proposition}

To prove this proposition, we need the elementary fact that standard bit mutation hardly reduces $d(\cdot)$ by 50\% or more.

\begin{proposition}\label{prop:mut}
  Let $p \in [0,1]$, $x \in \{0,1\}^n$ with $d:=d(x) \le 0.6n$, and $y$ be obtained from flipping each bit of $x$ independently with probability $p$. Then $\Pr[d(y) \le 0.5 d] = \exp(-\Omega(d))$.
\end{proposition}

\begin{proof}
  Let first $0.1n \le d \le 0.6n$. Then $E[d(y)] \ge \min\{d,0.4n\}$ regardless of $p$. Consequently, $\Pr[d(y) \le d/2] \le \Pr[d(y) \le E[d(y)] - 0.05n] \le \exp(-\Theta(n))$ by the additive Chernoff bound (Theorem~\ref{thm:chernoff}~\ref{addchernoff}). 
  
  Let now $d \le 0.1n$. Let $g := g(x,y)$ and $b = b(x,y)$. Trivially, we have $d(y) = d - g + b$. Let first $p \le 1/4$. Since $g$ is binomially distributed with parameters $d$ and $p$, we have $E[g] = dp \le d/4$ and $\Pr[g \ge d/2] \le \exp(-\Omega(d))$ by the multiplicative Chernoff bound (Theorem~\ref{thm:chernoff}~\ref{multchernofflower}). We thus have $\Pr[d(y) \le d/2] \le \Pr[g \ge d/2] \le \exp(-\Omega(d))$. Let now $p \ge 1/4$. Then $E[b] = (n-d)p \ge 0.225n$ and $\Pr[b \le 0.1n] \le \exp(-\Omega(n))$. Since trivially $g \le d \le 0.1n$, we have $\Pr[d(y) \le d/2] \le \Pr[b \le 0.1n] \le \exp(-\Omega(n))$. 
\end{proof}

\begin{proof}[Proof of Proposition~\ref{prop:lbsmallsteps}]
  To ease the calculations, we use the following Gedankenexperiment. Imagine that the \ga does not select a winning individual $x'$ at the end of the mutation phase, but instead creates $\lambda$ crossover offspring from each of the $\lambda$ mutation offspring. Clearly, the set of $\lambda$ crossover offspring from a true run of the algorithm is contained in this set of $\lambda^2$ offspring. Hence it suffices to show that none of the $\lambda^2$ offspring from the Gedankenexperiment and none of the $\lambda$ mutation offspring has a fitness distance of $d/2$ or better. 
  
  Let $\tilde y$ be a crossover offspring of the Gedankenexperiment. Let $\tilde x$ be the mutation offspring that was used in the crossover giving rise to $\tilde y$. Then $\tilde x$ is obtained from $x$ by flipping each bit independently with probability $k/n$---the \ga creates $\tilde x$ algorithmically different, namely by first sampling $\ell$ and then flipping $\ell$ bits, but the result is that $\tilde x$ has the distribution described above due to the choice of $\ell$. Now $\tilde y$ is obtained from a biased crossover of $x$ and $\tilde x$. Since each bit of $\tilde x$ makes it into $\tilde y$ only with probability $r/k$, we see that we have $\tilde y_i \neq x_i$ with probability $(k/n) \cdot (r/k) = r/n$ independently for all $i \in [1..n]$. Consequently, $\tilde y$ has the same distribution as if it was generated from $x$ by standard bit mutation with mutation rate $r/n$. 
  
  Since all mutation and crossover offspring are distributed as if generated via standard bit mutation (with some mutation rate that does not matter here), Proposition~\ref{prop:mut} and a simple union bound over the $\lambda(\lambda+1)$ mutation and crossover offspring  shows that with probability at least $1 - \lambda(\lambda+1) \exp(-\Omega(d))$ none of these has a fitness distance of $d/2$ or better.   
\end{proof}

The second proposition shows that, apart from exceptional cases, we can freely switch between the number of iterations $T$ and the number of fitness evaluations $F$ needed to find an optimum. This is a well-known fact, so we give its proof merely for reasons of completeness. Recall that the optimization time is defined to be the number of fitness evaluations until for the first time an optimal solution is evaluated. Consequently, if say the first mutation offspring by chance is an optimal solution, then the optimization time $F$ would be $2$. The number of iterations $T$, though, would be $1$, so the estimate $F = \Omega(\lambda T)$ is not valid. The following lemma shows this exceptional case only occurs for $E[T] < 2$, so that usually we can (and will without further notice) use the argument $E[F] = \Omega(\lambda E[T])$.

\begin{proposition}\label{prop:lbtf}
  If $E[T] \ge 2$, then $E[F] = \Theta(\lambda E[T])$.
\end{proposition}

\begin{proof}
  By definition of $F$ and $T$, we have $T = \lceil (F-1)/2\lambda \rceil \le (F-1)/2\lambda + 1$. Consequently, $F \ge 2(T-1)\lambda+1$ and $E[F] \ge E[2(T-1)\lambda+1] \ge 2(E[T]-1)\lambda \ge E[T] \lambda$ when $E[T] \ge 2$. Since $F \le 2\lambda T + 1$, we also have $E[F] = O(\lambda E[T])$.  
\end{proof}

We now start proving a number of lower bounds for the runtime of the \ga. They do not logically rely on each other. The first result shows that, unless $\lambda$ is excessively large, the expected optimization time is at least $\Omega(\lambda)$. 

\begin{lemma}\label{lem:lblambdalarge}
  $E[F] \ge \min\{\lambda,2^{n}\}/2$.
\end{lemma}

\begin{proof}
  The proof builds on the following simple observation: Let $\tilde x$ be a mutation offspring generated in the first iteration. Then $\tilde x$ is uniformly distributed in $\{0,1\}^n$. Indeed, let $x$ be the random initial search point, which is uniformly distributed in $\{0,1\}^n$, which is equivalent to saying that each $x_i$ independently is equal to $1$ with probability $1/2$ (and is equal to $0$ otherwise). Now $\tilde x$ is generated from $x$ by flipping each bit independently with probability $k/n$. Consequently, the bits of $\tilde x$ are independent. We also compute $\Pr[\tilde x_i = 1] = \Pr[x_i = 0] (k/n) + \Pr[x_i = 1] (1 - k/n) = 1/2$. Hence $\tilde x$ is uniformly distributed in $\{0,1\}^n$.

  With this preliminary consideration, the proof of the lemma is very easy. Let $L$ be a non-negative integer. Let $x_0, x_1, \dots, x_{L}$ be the initial random search point and the first $L$ mutation offspring. Note that each of these search points individually is uniformly distributed in $\{0,1\}^n$. Consequently, by a simple union bound, the probability that one of these search points is the optimum is at most $(L+1) 2^{-n}$. In other words, the number $F$ of fitness evaluations until an optimal solution is found, satisfies $\Pr[F \ge L+2] \ge 1 - (L+1) 2^{-n}$ for all $0 \le L \le \lambda$. By Lemma~\ref{lem:integerexpectation}, taking $K = \min\{\lambda+1,2^{n}\}$, we compute $E[F] = \sum_{i=1}^\infty \Pr[F \ge i] \ge \sum_{i=1}^{K} \Pr[F \ge i] \ge \sum_{i=1}^{K} (1 - (i-1) 2^{-n}) = K - \frac{K(K-1)}{2} 2^{-n} = K (1-2^{-n-1}(K-1)) \ge \min\{\lambda,2^n\}/2$. 
\end{proof}

We proceed by regarding the case that $k$ is large, say $k \ge \sqrt n$. While this is much larger than all values of $k$ that lead to the optimal expected runtime, the proof is not very simple. The reason is that even such large values for $k$ can give a near-optimal runtime of $O(n \log n)$ for suitable choices of the other parameters, e.g., small values for $\lambda$ and $r = 1$ (we do not prove this statement). 

\begin{lemma}\label{lem:lblargek}
  If $k \ge \sqrt n$ and $\lambda = \exp(o(n^{1/16}))$, then \[E[F] = \min\{\Omega(r^{-1} \exp(\Theta(r)) n \log n), \exp(\Omega(r)) n^2 \log n, \exp(\Omega(n^{1/16}))\},\] which attains its asymptotically optimal value $\Omega(n \log n)$ for $r = \Theta(1)$.  
\end{lemma}

\begin{proof}
  We start by analyzing the progress the \ga makes in one iteration starting with a search point $x$ having fitness distance $d := d(x) \in [n^{3/4},n^{7/8}]$. More precisely, denote by $z$ an individual with maximal fitness among all mutation and crossover offspring generated in this iteration and among the parent $x$. Needless to say, $z$ can be the parent $x$, the crossover winner $y$, or the mutation winner $x'$. To use drift analysis, we shall regard the progress $d(x) - d(z)$. Note that this is $0$ if $\max\{f(x'),f(y)\} \le f(x)$. Note also that $d(x)-d(z) = f(z)-f(x)$.
  
  Let $\tilde x$ be a mutation offspring. Let $\tilde g = g(x,\tilde x)$  be the number of good bits of~$\tilde x$. Since $\tilde g$ follows a binomial distribution with parameters $d$ and $k/n$, we have $E[\tilde g] = dk/n \ge n^{1/4}$ and $\Pr[\tilde g \ge 2dk/n] \le \exp(-(dk/n)/3) \le \exp(-\Omega(n^{1/4}))$. Hence only with probability at most $\lambda \exp(-\Omega(n^{1/4}))$, there is a mutation offspring with at least $2dk/n$ good bits; in this rare case we estimate the progress $f(z) - f(x)$ via the trivial bound $f(z) - f(x) \le n$. Similarly, in the exceptional case that $\ell < k/2$, which occurs with probability at most $\exp(-\Omega(k)) \le \exp(-\Omega(n^{1/2}))$, we again estimate $f(z) - f(x) \le n$.
  
  Hence let us now analyze the progress in the \emph{regular situation} that no mutation offspring has $2dk/n$ good bits or more (and thus $g(x,x') < 2dk/n$) and that $\ell \ge k/2$  (and thus $x'$ has at least $b(x,x') \ge \ell - (2dk/n) \ge (k/2)-(2dk/n) = k ((1/2) - 2n^{-1/8}) \ge k/4$ bad bits). Since $b(x,x') > g(x,x')$, we have $z \neq x'$, so it remains to analyze the crossover offspring. Consider an offspring $\tilde y$ generated in the crossover phase. 
  
  Let us consider first \emph{the case that $r \ge n^{1/16}$}. Then $\tilde b := b(x,x',\tilde y)$ satisfies $E[\tilde b] \ge (k/4) \cdot (r/k) = r/4$. Hence with probability $1 - \exp(-\Omega(r)) \ge 1 - \exp(-\Omega(n^{1/16}))$, the crossover offspring $\tilde y$ has taken at least $k/8$ bad bits from $x'$. This is more than the number of good bits $x'$ has, so regardless of how many good bits make it into $\tilde y$, we have $f(\tilde y) \le f(x)$. Consequently, with probability $1 - \lambda \exp(-\Omega(n^{1/16}))$, no crossover offspring has a fitness better that $x$, and hence $f(z) = f(x)$. For the remaining probability $\lambda \exp(-\Omega(n^{1/16}))$, we estimate $f(z) - f(x) \le n$. In total, if $r \ge n^{1/16}$, we have $E[f(z)-f(x)] \le n \lambda \exp(-\Omega(n^{1/16})) = \lambda \exp(-\Omega(n^{1/16}))$.
  
  We now turn to \emph{the case that $r < n^{1/16}$}. In this case,  $\tilde g := g(x,x',\tilde y)$ satisfies $E[\tilde g] \le (2dk/n) \cdot (r/k) = 2dr/n \le 2n^{-1/16}$. We regard separately the situations that $\tilde g = 0$, $\tilde g \in [1..47]$, $\tilde g \in [48..\lfloor E[\tilde b]/2 \rfloor]$, and $\tilde g \ge E[\tilde b]/2$. Clearly, when $\tilde g = 0$, we have $f(\tilde y) \le f(x)$. Markov's inequality shows that good bits exist only with probability $E[\tilde g] \le 2dr/n$, hence, $\Pr[\tilde g \in [1..47]] \le \Pr[\tilde g \ge 1] \le 2dr/n$. Conditioning on $\tilde y$ having between one and $47$ good bits, we trivially observe $f(\tilde y) - f(x) \le 47$. However, for $\tilde y$ to have a fitness better than $f(x)$, it is necessary (but not sufficient) that at most $46$ bad bits are copied from $x'$ to $\tilde y$. The probability of this event, which is independent of any event regarding good bits only, is at most $\exp(-\Omega(E[\tilde b])) \le \exp(-\Theta(r))$, because the expected number $E[\tilde b]$ of bad bits copied into $\tilde y$ is $\Theta(r)$. By Theorem~\ref{thm:chernoff}~\ref{multchernoffupperstrong}, the probability that $48$ or more good bits are copied into $\tilde y$ is $O(n^{-3})$, hence $\Pr[\tilde g \in [48..\lfloor E[\tilde b]/2 \rfloor]] \le \Pr[\tilde g \ge 48] = O(n^{-3})$. In this situation, for $f(\tilde y)$ to be larger than $f(x)$, we need $\tilde b < \tilde g \le E[\tilde b]/2$, which happens with probability $\exp(-\Omega(E[\tilde b])) \le \exp(-\Omega(r))$. Finally, if $E[\tilde b]/2 \ge 48$, then the probability that $\tilde g \ge E[\tilde b]/2$ is at most $n^{-3-\Omega(E[\tilde b])} \le n^{-3-\Omega(r)}$ by Theorem~\ref{thm:chernoff}~\ref{multchernoffupperstrong}. Hence
\begin{align*}
  E[\max\{f(\tilde y) - f(x),0\}] & \le \Pr[\tilde g = 0] \cdot 0\\
  & \quad + \Pr[\tilde g \in [1..47]] \exp(-\Omega(r)) \cdot 47\\
  & \quad + \Pr[\tilde g \in [48..\lfloor E[\tilde b]/2 \rfloor]] \exp(-\Omega(r)) \cdot n\\
  & \quad + \Pr[\tilde g \ge E[\tilde b]/2 \mid E[\tilde b]/2 \ge 48] \cdot n\\
  & \le 0 + \tfrac{2dr}{n} \exp(-\Omega(r)) \cdot 47 + O(n^{-3}) \exp(-\Omega(r)) \cdot n + n^{-3-\Omega(r)} \cdot n\\
  & \le O((\tfrac{dr}{n} + n^{-2}) \exp(-\Omega(r)).
\end{align*}
  
   Since $y$ is chosen among the crossover offspring $\tilde y$ such that $f(\tilde y)$, and equivalently, $f(\tilde y) - f(x)$ is maximal, we have $f(y) - f(x) \le \sum_{\tilde y} \max\{f(\tilde y) - f(x),0\}$, where $\tilde y$ runs over all $\lambda$ crossover offspring. Consequently, $E[f(y) - f(x)] = O(\lambda (\frac{dr}{n} +n^{-2}) \exp(-\Omega(r)))$.
  
  Taking the two cases regarded separately together, we see that for any $r$ we have $E[f(z) - f(x)] = E[\max\{0,f(y) - f(x)\}] = \max\{O(\lambda (\frac{dr}{n}+n^{-2}) \exp(-\Omega(r))), \lambda \exp(-\Omega(n^{1/16}))\}$, when we condition on being in the regular situation. In the general situation, we have $E[f(z) - f(x)] = \lambda \exp(-\Omega(n^{1/4}) n + (1 - \lambda \exp(-\Omega(n^{1/4})) \max\{O(\lambda (\frac{dr}{n}+n^{-2}) \exp(-\Omega(r))), \lambda \exp(-\Omega(n^{1/16}))\} = \max\{O(\lambda (\frac{dr}{n}+n^{-2}) \exp(-\Omega(r))), \lambda \exp(-\Omega(n^{1/16}))\}$. To ease the following multiplicative drift argument, we estimate this bluntly by $E[f(z)-f(x)] \le \max\{O(\lambda (\frac{dr}{n}+dn^{-2})) \exp(-\Omega(r))), d\lambda \exp(-\Omega(n^{1/16}))\} = d \max\{O(\lambda \max\{r,n^{-1}\} \exp(-\Omega(r)) / n), \lambda \exp(-\Omega(n^{1/16}))\}$. 
  
  Building on this drift statement, we now use Witt's lower bound result for multiplicative drift (Theorem~\ref{thm:multidriftlower}). Consider a run of the \ga. For $t = 0, 1, \dots$, denote by $x_t$ the search point $x$ at the beginning of the $(t+1)$st iteration except if before that once an optimal solution was generated, in this case let $x_t$ be any optimal solution. By Proposition~\ref{prop:lbsmallsteps}, with probability at least $1 - \lambda(\lambda+1) \exp(-\Omega(n^{7/8}))$ the \ga at some time $t_0$ reaches a search point $x_{t_0}$ with $d(x_{t_0}) \in [0.5 n^{7/8},n^{7/8}]$. We show that in this case, we have an expected optimization time as claimed, which implies that also the unconditioned expectation is of the same order of magnitude. 
  
  For $t = 0, 1, \dots$ define $X_t = \max\{d(x_{t_0+t}),1\}$. Observe that $X_{t+1} \le X_t$ for all $t \ge 0$. Let $s_{\min} := n^{3/4}$. Then we have shown above that if $X_t = s > s_{\min}$, then $E[X_t - X_{t+1}] \le s \max\{K_1 \lambda \max\{r,n^{-1}\} \exp(-K_2 r) / n, \lambda \exp(-K_3 n^{1/16}))\}$ for some absolute constants $K_1, K_2, K_3$. Note that the drift of the process $X_t$ might be smaller than this, because above we took $z$ as the best individual among parent and all individuals generated in the iteration. The first condition of the drift theorem thus is fulfilled with $\delta = \max\{K_1 \lambda \max\{r,n^{-1}\} \exp(-K_2 r) / n, \lambda \exp(-K_3 n^{1/16}))\}$. From Proposition~\ref{prop:lbsmallsteps} we know that $\Pr[X_{t+1} \le s/2] \le \lambda(\lambda+1) \exp(-\Omega(s))=\exp(-\Omega(s))$. Hence for $n$ (and thus also $s$) sufficiently large, also the second condition of the drift theorem is satisfied (with $\beta = 1/2$); also we have $E[T] = \Omega(\log n)$ to enable the argument $E[F] = \Omega(\lambda E[T])$ below. We may thus apply the theorem and derive that the first $t$ such that $X_t \le s_{\min}$ satisfies $E[t] = \Omega(\frac{\ln(X_0) - \ln(s_{\min})}{\delta}) = \Omega(\min\{\frac{\exp(\Theta(r) n \log n}{\max\{r,1/n\} \lambda}), \exp(\Omega(n^{1/16}))/\lambda\})$. Note that this, naturally, is a lower bound on $E[T]$. Consequently, $E[F] = \Omega(\lambda E[T]) = \Omega(\min\{\frac{\exp(\Omega(r))}{r} n \log n, \exp(\Omega(r)) n^2 \log n, \exp(\Omega(n^{1/16}))\}) $.
\end{proof}

The following lower bound imitates the classic argument that if in all applications of the mutation operators not enough bits are flipped, then there will be a bit that is initially zero and that was never touched in a mutation operation. The proof is slightly more involved as usual for this type of argument, because our mutation operator uses a hypergeometric distribution.

\begin{lemma}\label{lem:lbcoupon}
  Let $0 < k \le n/12$ and $k \lambda = o(n \log n)$. Let $\alpha < 1/4$. Let $t = \lfloor\alpha n \ln(n) / (k\lambda) \rfloor$. Then $\Pr[T \le t] = \exp(-\Omega(\min\{k t, n^{1-4\alpha}\}))$. In particular, $E[F] = \Omega(\frac{n \log n}{k})$. Consequently, an optimal parameter setting satisfies $k = \Omega(\sqrt{\log(n) \log\log(n) / \log\log\log(n)}) = \Omega(\lambda^*)$.
\end{lemma}

\begin{proof}
  Using the Chernoff bound of Theorem~\ref{thm:chernoff} \ref{multchernofflower}, we see that with probability $1 - \exp(-\Omega(n))$, the initial search point has at least $n/3$ bits valued zero (``missing bits''). 
  
  Let us consider what happens in the first $t = \lfloor \alpha n \ln(n) / (k\lambda) \rfloor$ iterations. Denote by $\ell_1, \ldots, \ell_t$ the values of $\ell$ chosen by the algorithm in these iterations. Note that the $\ell_i$ are independent random variables each having a binomial distribution with parameters $n$ and $k/n$. Consequently, $L := \sum_{i = 1}^t \ell_i$ is a sum of $tn$ independent $0,1$ random variables that are one with probability $k/n$. Hence we have $E[L] = t k$. By the multiplicative Chernoff bound of Theorem~\ref{thm:chernoff}~\ref{multchernoffupper}, we see that with probability $1 - \exp(-\Omega(t k))$, we have $L \le 2tk$.
  
  Again exploiting the binomial distribution of the $\ell_i$, we derive from Theorem~\ref{thm:chernoff}~\ref{chernoffzweihoch} that $\Pr[\ell_i \ge n/2] \le 2^{-n/2}$; note that here we used that $k \le n/12$ and thus $E[\ell_i] = k \le n/12$. Consequently, with probability $1 - \exp(-\Omega(n))$, all $\ell_i$ are at most $n/2$ (union bound). 
  
  In the following, we condition on none of these three rare events occurring. More precisely, we condition on that there are at least $n/3$ missing bits and we condition on a particular outcome of the $\ell_i$ that avoids the exceptional events $L > 2tk$ and $\ell_i > n/2$ for some $i \in [1..t]$. The probability that a particular one of the missing bits is never flipped in the mutation phases of the first $t$ iterations is 
  \begin{align*}
  \prod_{i = 1}^t (1 - \ell_i / n)^\lambda 
  & \ge \prod_{i = 1}^t \exp(- 2\ell_i / n)^\lambda 
  = \exp(- 2 \lambda L / n) \ge \exp(-4k\lambda t / n) 
  \ge n^{-4\alpha},
  \end{align*}
  where we have used in the first step that $1-c \ge e^{-2c}$ for $0 \le c \le 1/2$.
  
  Denote by $M \subseteq [1..n]$ the set of missing bits and by $A_i$ the event that bit $i$ was flipped at least once in some mutation step in the first $t$ iterations. Then we just showed $\Pr[A_i] \le 1 - n^{-4\alpha}$. We want to show that it is very unlikely that all events $A_i$ are fulfilled. 
  
  Unfortunately, the events $A_i, i \in M$, are not independent, since already in a single application of the mutation operator the bits are not treated independently, but according to a hypergeometric distribution. We therefore now show that they satisfy the following negative correlation property: \[\forall I \subseteq M : \Pr\bigg[\bigcap_{i \in I} A_i\bigg] \le \prod_{i \in I} \Pr[A_i].\]  
  We proceed via induction over the cardinality of $I$. For $|I|= 0, 1$, there is nothing to show. Let $I \subseteq M$ such that $|I| \ge 2$. Let $j \in I$ and $I' := I \setminus \{j\}$. Then 
\begin{equation}
  \Pr\bigg[\bigcap_{i \in I'} A_i\bigg] = \Pr\bigg[\bigcap_{i \in I'} A_i \bigg| A_j\bigg]\Pr[A_j] + \Pr\bigg[\bigcap_{i \in I'} A_i \bigg| \bar A_j\bigg] \Pr[\bar A_j].\label{eq:condi}
\end{equation}  
  It is clear that $\Pr[\bigcap_{i \in I'} A_i \mid \bar A_j]$ is at least as large as $\Pr[\bigcap_{i \in I'} A_i]$---conditioning on $\bar A_j$ is equivalent to saying that the random subsets of bits to be flipped are not chosen as subsets of $[1..n]$, but of $[1..n] \setminus \{j\}$, and this increases the probability of the event $\bigcap_{i \in I'} A_i$. More formally, there is the following coupling from the unconditioned probability space into the one conditional on $\bar A_j$. Whenever in the unconditioned probability space the $j$-th bit is flipped in some iteration, we replace this bit-flip by flipping a new bit different from $j$ and the other bits flipped in this iteration. This is exactly the random experiment done in the probability space conditional on $\bar A_j$. Clearly, if the event $\bigcap_{i \in I'} A_i$ holds in the unconditioned space, this is not affected by the coupling. Hence the probability of the event $\bigcap_{i \in I'} A_i$ is not smaller in the space conditional on $\bar A_j$. 
  
  Since thus $\Pr[\bigcap_{i \in I'} A_i \mid \bar A_j] \ge \Pr[\bigcap_{i \in I'} A_i]$, we see from equation~(\ref{eq:condi}) that $\Pr[\bigcap_{i \in I'} A_i \mid  A_j] \le \Pr[\bigcap_{i \in I'} A_i]$. From $\Pr[\bigcap_{i \in I'} A_i \mid A_j] = \Pr[\bigcap_{i \in I} A_i] / \Pr[A_j]$ we derive the desired statement $\Pr[\bigcap_{i \in I} A_i ] \le \Pr[\bigcap_{i \in I'} A_i] \Pr[A_j]$. Applying induction to $I'$, we have $\Pr[\bigcap_{i \in I} A_i ] \le \prod_{i \in I'} \Pr[A_i] \Pr[A_j] = \prod_{i \in I} \Pr[A_i]$.
  
  Using this negative correlation property for the set of all missing bits, we conclude that the probability $\Pr[\bigcap_{i \in M} A_i]$ that all missing bits were flipped at least once, is  $\Pr[\bigcap_{i \in M} A_i] \le \prod_{i \in M} \Pr[A_i] \le (1 - n^{-4\alpha})^{n/3} \le \exp(-n^{-4\alpha})^{n/3} = \exp(-n^{1-4\alpha}/3)$, where we used the estimate $(1+x) \le e^x$ valid for all $x \in \R$.
  
  Consequently, with probability at least $1 - \exp(-\Omega(n)) - \exp(-\Omega(k t)) - \exp(-\Omega(n^{1-4\alpha}))$, there is a bit that initially has the value zero and is not flipped in the first $t$ iterations, implying that the \ga needs more than $t$ iterations to generate the optimum as mutation or crossover offspring. This high-probability statement immediately implies the claimed bound on the expected optimization time, using again $E[T] \ge 2$ and $E[F] = \Theta(\lambda E[T])$.
\end{proof}

\begin{lemma}\label{lem:lbsmallk}
  If $\omega(1) = k \le \sqrt n$, then $E[F] = \Omega(n \log n \min\{\tfrac{\exp(\Omega(r))}{\lambda r}, \tfrac{n^3}{\lambda}, \tfrac{\exp(\Omega(k))}{k}\})$.  
\end{lemma}

\begin{proof}
  We first analyze the progress made in an iteration starting with a search point with fitness distance between $n^{1/8}$ and $n^{1/4}$ and then use this information with the lower bound multiplicative drift theorem to obtain the claimed lower bound for the optimization time.
  
  Consider an iteration starting with a search point $x$ with $n^{1/8} \le d(x) \le n^{1/4}$. Let $z$ be a search point among $\{x, x', y\}$ with maximal fitness. We aim at estimating the expected progress $E[d(x) - d(z)] = E[f(z) - f(x)]$. Since $\ell$ is binomially distributed, we have $\ell < k/2$ with probability at most $\exp(-\Omega(k))$ by the multiplicative Chernoff bound. Similarly, with probability at most $\exp(-\Omega(k))$, we have $\ell > 2k$. In this case, we have $E[\ell | \ell > 2k] \le 3k+1$ by Lemma~\ref{lem:condbinomial}. Hence $E[\ell | \ell \notin [k/2,2k]] = O(k)$. 
  
  Let $\tilde x$ be an offspring created in the mutation phase. Let $\tilde g := g(x,\tilde x)$. Conditioning on the outcome of $\ell$, $\tilde g$ has a hypergeometric distribution with parameters $n$, $\ell$, and $d$. Hence $E[\tilde g] = d\ell/n$. For the mutation winner $x'$, note that $g' := g(x,x') \le \sum_{i = 1}^\lambda g(x,x^{(i)})$. Hence $E[g'] \le \lambda d \ell/n$.
    
  For $\ell \notin [k/2,2k]$, we use the estimate that $f(z) - f(x) \le g'$ with probability one (note that this estimate is fulfilled both for $z=x'$ and $z=y$). Hence we compute 
  \begin{align*}
  E[f(z) - f(x) \mid \ell \notin [k/2,2k]] &= \sum_{i \notin [k/2,2k]} \Pr[\ell = i \mid \ell \notin [k/2,2k]]\,  E[f(z) -f(x) \mid \ell = i] \\
  &\le \sum_{i \notin [k/2,2k]} \Pr[\ell = i \mid \ell \notin [k/2,2k]] \, E[g' \mid \ell = i] \\
  &\le \sum_{i \notin [k/2,2k]} \Pr[\ell = i \mid \ell \notin [k/2,2k]] \, \lambda d i / n \\
  &=  E[\ell  \mid \ell \notin [k/2,2k]] \, \lambda d /n = O(k \lambda d / n).   
  \end{align*}
  
   Hence let us assume (and condition on) that $k/2 \le \ell \le 2k$. Then $E[\tilde g] = d\ell/n \le 2n^{-1/4}$ and thus $\Pr[\tilde g \ge 20] \le O(n^{-5})$ by Theorem~\ref{thm:chernoff}~\ref{multchernoffupperstrong} and Theorem~\ref{thm:hypergeomchernoff}. Similarly, $E[g'] = \lambda d\ell/n \le 2 \lambda d k / n$ and the probability that $x'$ has a good bit at all is $\Pr[g' \ge 1] \le E[g'] = 2 \lambda dk/n$ by Markov's inequality. If $g'=0$, then $f(x) = f(z)$. So let us consider the case that $g' > 0$. Without conditioning on $g' > 0$, we have $\Pr[g' \ge 20] \le \lambda \Pr[\tilde g \ge 20] = O(\lambda n^{-5})$. Hence conditional on $g' > 0$, this probability is at most $O(\lambda n^{-5} / \Pr[g' \ge 1]) = O(\lambda n^{-5} / \min\{1,2\lambda d  k / n\}) = O(\max\{\lambda n^{-5}, n^{-4}/(dk)\})$. In this rare event, we can safely estimate $f(z) -f(x) \le n$, so let us turn to the more interesting case that $1 \le g' < 20$. Since $H(x,x') = \ell$, we have $b(x,x') \ge \ell - 19 $. Consequently, $\ell = \Theta(k) = \omega(1)$ implies that no mutation offspring can be better than $x$. Let $\tilde y$ be an offspring generated in the crossover phase. Let $b_c := b(x,x',\tilde y)$ denote the number of bad bits of $x'$ that make it into $\tilde y$. For $f(\tilde y) > f(x)$ to hold, we need that $b_c \le 19$, but also that at least one good bit makes it into $\tilde y$, that is, $g(x,x',\tilde y) \ge 1$. Since $b_c$ follows a binomial distribution with parameters $b(x,x')$ and $r/k$, we have $E[b_c] = b(x,x') r/k \ge (\ell-19)r/k$. Hence $\Pr[b_c \le 19] \le \exp(-\Omega(r))$ by the multiplicative Chernoff bound. The expected number  of good bits making it into~$\tilde y$ is at most $E[g(x,x',\tilde y)] \le 19 \cdot (r/k)$, hence by Markov's inequality this is also an upper bound for the probability that good bits make it into $\tilde y$ at all. Putting all this together and taking a union bound over the $\lambda$ crossover offspring, we see that (still in the case that $1 \le g' \le 19$) the probability that some crossover offspring is better than $x$ is at most $\lambda \cdot (19r/k) \cdot \exp(-\Omega(r))$;  only then we have $f(z) > f(x)$, however, the gain is at most $19$. Consequently, $E[f(z) - f(x) \mid k/2 \le \ell \le 2k \wedge 1 \le g' \le 19] \le \Pr[f(z) > f(x) \mid k/2 \le \ell \le 2k \wedge 1 \le g' \le 19] \cdot 19 \le 19 \lambda (19r/k) \exp(-\Omega(r))$. 
  
  We thus have 
  \begin{align*}
    E[f&(z) - f(x) \mid k/2 \le \ell \le 2k] \\
    & = \Pr[g' \ge 1] E[f(z) - f(x) \mid k/2 \le \ell \le 2k \wedge g' \ge 1] \\
    & = (\lambda d k / n) \big(\Pr[g' \le 19 \mid g' \ge 1] E[f(z) - f(x) \mid k/2 \le \ell \le 2k \wedge 1 \le g' \le 19] \\
    & \quad + \Pr[g' \ge 20 \mid g' \ge 1] E[f(z) - f(x) \mid k/2 \le \ell \le 2k \wedge g' \ge 20]\big) \\
    & \le (\lambda d k / n) \big(1 \cdot 19^2 \lambda r \exp(-\Omega(r)) /k ) + O(\max\{\lambda n^{-5}, n^{-4}/(dk)\}) n \big)\\
    &\le O(\lambda^2 d r \exp(-\Omega(r)) n^{-1} + \lambda^2 dk n^{-5} + \lambda n^{-4}) \\
    & = O(d \lambda^2 (r \exp(-\Omega(r)) n^{-1} + n^{-4})). 
  \end{align*}
  
  Together with the exceptional case that $\ell \notin [k/2,2k]$, we obtain
  \begin{align*}
    E[f&(z) - f(x)] \\
    & = \Pr[k/2 \le \ell \le 2k] E[f(z) - f(x) \mid k/2 \le \ell \le 2k] \\
    & \quad + \Pr[\ell \notin [k/2,2k]] E[f(z) - f(x) \mid \ell \notin [k/2,2k]] \\
    & = O(d \lambda^2 (r \exp(-\Omega(r) n^{-1} + n^{-4}) + \exp(-\Omega(k)) O(k \lambda d / n)\\
    & = O(\tfrac{d \lambda}{n} (\lambda r \exp(-\Omega(r)) + \lambda n^{-3} + k \exp(-\Omega(k)))).
  \end{align*}
  
  We now use the lower bound multiplicative drift theorem (Theorem~\ref{thm:multidriftlower}) to prove our claim. By Proposition~\ref{prop:lbsmallsteps}, with high probability a run of the \ga once encounters a search point $x_0$ with $d(x_0) \in [0.5n^{1/4}, n^{1/4}]$. For this case, we give a lower bound for the expected optimization time (which implies asymptotically the same bound for the general case). Denote by $x_t$, $t \ge 0$, the sequence of search points $x$ generated by the \ga in the sequel (except that $x_t$ is the optimum solution from the first point on that the optimum was found). Let $s_{\min} := n^{1/8}$. We just showed that $E[d(x_{t+1}) - d(x_t) | d(x_t) = s] \le s \delta$ holds for all $s \in [s_{\min},d(x_0)]$, where we set $\delta = K \tfrac{\lambda}{n} (\lambda r \exp(-\Omega(r)) + \lambda n^{-3} + k \exp(-\Omega(k)))$ for some absolute constant $K$. By Proposition~\ref{prop:lbsmallsteps} again, we know that 
\begin{equation}
  \Pr[d(x_t) - d(x_{t+1}) \ge 0.5 s \mid d(x_t) = s] \le \lambda(\lambda+1) \exp(-s) \le 0.5 \delta / \ln(s) \label{eq:t2}.
\end{equation}
  Consequently, we may apply Theorem~\ref{thm:multidriftlower} to the random process $(\max\{1,d(x_t)\})_{t \ge 0}$, and learn that the expected first $t$ such that $d(x_t) \le s_{\min}$ is $\Omega(\log(n) / \delta) = \Omega(\frac{n \log n}{\lambda (\lambda r \exp(-\Omega(r)) + \lambda n^{-3} + k \exp(-\Omega(k)))})$. Consequently, $E[T]$ is at least this number. By~\eqref{eq:t2}, we also have $E[T] \ge 2$ and thus $E[F] = \Omega(\lambda E[T]) = \Omega(\frac{n \log n}{\lambda r \exp(-\Omega(r)) + \lambda n^{-3} + k \exp(-\Omega(k))}) = \Omega(n \log n \min\{\tfrac{\exp(\Omega(r))}{\lambda r}, \tfrac{n^3}{\lambda}, \tfrac{\exp(\Omega(k))}{k}\})$.  
\end{proof}

The following result exploits the simple fact that if in one iteration a mutation strength of $\ell$ was used, then regardless of the population size no progress of more than $\ell$ can be made. 

\begin{lemma}\label{lem:lblambdak}
 Let $k \le n/4$. Then $E[F] = \Omega(\frac{n \lambda}{k})$.
\end{lemma}

\begin{proof}
  Let $x_0$ be the random initial search point. When $x_t$ is defined for some $t \ge 0$, let $x_{t+1}$ be the value of $x$ after one iteration of the \ga starting with $x=x_t$, unless this iteration generated the optimal solution, in this case let $x_{t+1}$ be the optimal solution. Hence the sequence $(x_t)_t$ describes a typical run of the \ga until the point when an optimal solution was generated. In particular, $T = \min\{t \ge 0 \mid d(x_t)=0\}$.
  
  We use the simple argument that all offspring generated in one iteration have a Hamming distance of at most $\ell$ from the parent. Consequently,  $E[d(x_t) - d(x_{t+1})] \le E[\ell] = k$, regardless of whether $x_{t+1}$ is an optimal mutation offspring or the crossover winner. By the additive drift theorem (Theorem~\ref{thm:drift}), we have $E[T |x_0] \ge d(x_0) / k$. Since the expected distance of a random search point from the optimum is $n/2$, the law of total expectation gives $E[T] \ge E[d(x_0)] / k = n / 2k$. This is at least $2$, so by Proposition~\ref{prop:lbtf}, we have $E[F] = \Omega(\frac{n \lambda}{k})$.
\end{proof}

\begin{lemma}\label{lem:lblambda}
  Let $k \le n/80$, $\lambda \le \exp(k/120)$, $\lambda = \exp(o(n))$, and $\lambda = \omega(1)$. Then $E[F] = \Omega(\frac{n \lambda \log\log(\lambda)}{r \log \lambda})$. 
\end{lemma}

\begin{proof}
%
  We shall show that the expected fitness gain in an iteration started with a search point with fitness distance at most $n/10$, is $O(r \log \lambda / \log\log \lambda)$. Since the \ga by Proposition~\ref{prop:lbsmallsteps}, here we use the assumption $\lambda = \exp(o(n))$, with high probability reaches once a search point $x$ with $f(x) \in [n/20,n/10]$,  the claim follows from the additive drift theorem (Theorem~\ref{thm:drift}).
  
  To prove the drift condition, consider one iteration of the \ga started with a parent individual $x$ with $d(x) \le n/10$. Let $z$ be the value of $x$ after one iteration, or the optimal search point if it was found as a mutation offspring (hence, as mutation winner). We show that the expected fitness gain $f(z) - f(x)$ is at most $O(\log \lambda / \log\log \lambda)$. For this, we first argue that we can assume that $k/2 \le \ell \le 2k$. Indeed, we have 
\begin{align*}
  E[f(z) - f(x)]  = &\Pr[\ell < k/2] \, E[f(z) - f(x) \mid \ell < k/2]\\ 
  + &\Pr[k/2 \le \ell \le 2k] \, E[f(z) - f(x) \mid k/2 \le \ell \le 2k]\\ 
  + &\Pr[\ell > 2k] \, E[f(z) - f(x) \mid \ell > 2k].
\end{align*}
By the multiplicative Chernoff bounds of Theorem~\ref{thm:chernoff}, both $\Pr[\ell < k/2]$ and $\Pr[\ell > 2k]$ are $\exp(-\Omega(k))$. Since all offspring generated in one iteration (in either mutation and crossover phase) have Hamming distance at most $\ell$ from $x$, we immediately have $E[f(z) - f(x) \mid \ell < k/2] < k/2$. By Lemma~\ref{lem:condbinomial}, we also have $E[f(z) - f(x) \mid \ell > 2k] \le E[\ell \mid \ell > 2k] \le 3k+1$. Hence $E[f(z) - f(x)] \le k \exp(-\Omega(k)) + E[f(z) - f(x) \mid k/2 \le \ell \le 2k] \le O(1) + E[f(z) - f(x) \mid k/2 \le \ell \le 2k]$.

Hence we can assume for the remainder that $k/2 \le \ell \le 2k$. In this case, we argue as follows. Consider a mutation offspring $\tilde x$ and let $\tilde g := g(x,\tilde x)$. Then $E[\tilde g] = \ell d(x) / n \le \ell/10$. The probability that $\tilde g \ge \ell/5$ is at most $\exp(-(\ell/10)/3)) \le \exp(-k/60)$ by Theorem~\ref{thm:chernoff}~\ref{multchernoffupper}\footnote{To be precise, we use here the fact that the bound of Theorem~\ref{thm:chernoff}~\ref{multchernoffupper} is also valid if both occurrences of $E[X]$ are replaced by an upper bound for $E[X]$. This is a well-known fact, but seemingly a reference is not so easy to find. Hence the easiest solution is maybe to derive this fact right from Theorem~\ref{thm:chernoff}~\ref{multchernoffupper} by extending the sequence $X_1, \dots, X_n$ of random variables by random variables that take a certain value with probability one. By this, we can artificially increase $E[X]$ without changing the random variable $X - E[X]$. Hence the bound obtained from applying the Theorem to the extended sequence applies also to the original one.}  and Theorem~\ref{thm:hypergeomchernoff}. Since $\lambda \le \exp(k/120)$, we see that with probability at least $1 - \exp(-k/120)$, all mutation offspring have at most $\ell/5$ good bits, implying that $g' := g(x,x')$ satisfies $g' \le \ell/5$. Note that in the rare case that $g' > \ell/5$, which occurs with probability at most $\exp(-k/120)$, we still have $f(z) - f(x) \le g' \le \ell \le 2k$ with probability one, that is, this case contributes only another $k \exp(-\Omega(k))$  to the drift. 

Therefore, let us now also condition on $g' \le \ell/5$. Note that this also implies that  $b' := b(x,x')$ satisfies $b' \ge (4/5)\ell$, since all mutation offspring have Hamming distance exactly $\ell$ from the parent $x$. Consequently, all mutation offspring are worse than $x$, and $z \in \{x,y\}$.

We now analyze the result of a crossover phase. Consider a crossover offspring $y^{(j)}$ and let $g_j := g(x,x',y^{(j)})$. Then $E[g_j] \le g' r/k \le (\ell/5) \cdot (r/k) \le (2/5) r$. Let $\Delta = \frac{2r\ln(\lambda)}{\ln\ln(\lambda)} +s $ for a non-negative integer $s$. By Theorem~\ref{thm:chernoff}~\ref{multchernoffupperstrong}, 
\begin{align*}
  \Pr\Big[\max_{j \in [1..\lambda]} g_j \ge \Delta\Big] &\le \sum_{j=1}^\lambda \Pr[g_j \ge \Delta] \\
  &\le \lambda \bigg(\frac{e E[g_j]}{\Delta}\bigg)^{\Delta} \le \lambda \bigg(\frac{e\ln\ln(\lambda)}{5\ln(\lambda)}\bigg)^{2\frac{\ln(\lambda)}{\ln\ln(\lambda)}+s}  \le 2^{-s}. \end{align*}
  Consequently, by Lemma~\ref{lem:integerexpectation}, 
  \[E\Big[\max_j g_j\Big] = \sum_{t=1}^\infty \Pr\Big[\max_j g_j \ge t\Big] \le \frac{2r\ln \lambda}{\ln\ln \lambda} + \sum_{s=1}^\infty 2^{-s} \le \frac{2r\ln \lambda}{\ln\ln \lambda} + 1.\]
   Clearly, the number of surviving good bits is an upper bound on the progress $f(z)-f(x)$. Hence the expected progress of one iteration, conditional on the assumptions made before, is at most $E[f(z)-f(x) \mid k/2 \le \ell \le 2k \wedge |G'| \le \ell/5] \le \frac{2r\ln \lambda}{\ln\ln \lambda} + 1$. Since the drift is always bounded by $\ell \le 2k$, we have in fact $E[f(z)-f(x) \mid k/2 \le \ell \le 2k \wedge |G'| \le \ell/5] \le \min\{2k,\frac{2r\ln \lambda}{\ln\ln \lambda} + 1\}$. The unconditional drift thus is $E[f(z) - f(x)] \le \min\{2k,\frac{2r\ln \lambda}{\ln\ln \lambda} + 1\} + O(k) \exp(-\Omega(k)) = O(\min\{2k,\frac{2r\ln \lambda}{\ln\ln \lambda}\})$. The additive drift theorem (Theorem~\ref{thm:drift}), keeping in mind that we start with a search point with distance at least $n/20$, hence yields $E[T] = \Omega(\max\{(n/20)/2k, (n/20) \frac{\ln\ln \lambda}{2r\ln \lambda}\})$. This is at least 2, so we conclude $E[F] = \Omega(\lambda E[T]) \ge \Omega(n \frac{\lambda \ln\ln \lambda}{2r\ln \lambda})$.
\end{proof}

\begin{lemma}\label{lem:kvalue}
  Let $\lambda = \Theta(\lambda^*)$, $k = \exp(\omega(\sqrt{\log(n) \log\log\log(n) / \log\log(n)}\,))$, $k \le n/2$, and $r = \Theta(1)$. Then the expected runtime of the \ga with these parameters is $\omega(F^*)$.
\end{lemma}

\begin{proof}
  We first analyze the progress the \ga makes in one iteration starting with a search point $x$ having fitness distance $d := d(x) \in [3\ln\ln(n) n / k, n/3] =: [d_0,d_1]$. Let $z$ denote the new parent individual after one iteration (which is either $x$ or $y$), or the optimal solution in case one of the mutation offspring generated in this iteration was optimal. To use a lower bound drift theorem later, we prove an upper bound for $E[d(x) - d(z)]$. 
  
  We first convince ourselves that it is very unlikely that a mutation offspring is better than $x$. This will allow us to only regard the situation that $z \in \{x,y\}$. For a mutation offspring $\tilde x$ to be better than the parent $x$, more zero-bits have to flip than one-bits, that is, $\tilde g:= g(x,\tilde x) > b(x,\tilde x) =:\tilde b$. By a simple domination argument, we see that this event is most likely for $d(x) = n/3$, so let us assume this for the moment. Then $E[\tilde g] = k/3$ and $E[\tilde b] = 2k/3$. We have $\Pr[\tilde g \ge k/2] = \exp(-\Omega(k))$ and $\Pr[\tilde b \le k/2] = \exp(-\Omega(k))$. Consequently, $\Pr[f(\tilde x) \ge f(x)] \le \exp(-\Omega(k))$. We thus compute
  \begin{align*}
  E[f(z)-f(x)] &\le E\Big[\max_{\tilde x} \max\{0,f(\tilde x) - f(x)\}\Big] + E[\max\{0,f(y) - f(x)\}] \\
  &\le \lambda n \exp(-\Omega(k)) + E[\max\{0,f(y) - f(x)\}] \\
  &= O(n^{-2}) + E[\max\{0,f(y) - f(x)\}].
  \end{align*}
  
  We proceed by analyzing the quality of the crossover winner. Let $\tilde x$ be a mutation offspring and $x'$ be the winning individual of the mutation phase. Let $\tilde g = g(x,\tilde x)$ and $g' = g(x,x')$  be their numbers of good bits. We have $E[\tilde g] = dk/n$ and $\Pr[\tilde g \ge 2dk/n] \le \exp(-(dk/n)/3) \le 1 / \ln(n)$. Consequently,
\begin{align*}
 E[\max\{0,\tilde g-2dk/n\}] &\le (1/\ln(n)) E[\tilde g - 2dk/n \mid \tilde g \ge 2dk/n] \\
  &= dk/n\ln(n)
\end{align*}
by Lemma~\ref{lem:condbinomial}. We have 
\begin{align*}
		g' &\le \max_{\tilde x} g(x,\tilde x) \\
		&= (2dk/n) + \max_{\tilde x} (g(x,\tilde x)-2dk/n) \\
		& \le (2dk/n) + \sum_{\tilde x} \max\{0,g(x,\tilde x)-2dk/n\}
\end{align*} 
and thus $E[g'] \le 2dk/n + \sum_{\tilde x} E[\max\{0,g(x,\tilde x)-2dk/n\}] = 2dk/n + dk\lambda/n\ln(n) = (2+o(1))dk/n$, where all summations and maxima are taken over all mutation offspring.

Consider an offspring $\tilde y$ generated in the crossover phase. Let $\tilde g_{\tilde y} := g(x,x',\tilde y)$. Then $E[\tilde g_{\tilde g}] \le E[g'] \cdot (r/k) = O(d/n)$. Since the crossover winner $y$ is chosen among the crossover offspring $\tilde y$ such that $f(\tilde y)$, and equivalently, $d(x) - d(\tilde y)$, is maximal, we have $d(x) - d(y) \le \sum_{\tilde y} \max\{d(x) - d(\tilde y),0\}$, where $\tilde y$ runs over all $\lambda$ crossover offspring. Consequently, $E[d(x) - d(y)] = O(\lambda \frac{d}{n})$ and hence also $E[f(z)-f(x)] = O(\lambda \frac d n )$.


  Building on this drift statement, we now use Witt's lower bound result for multiplicative drift (Theorem~\ref{thm:multidriftlower}). Consider a run of the \ga. For $t = 0, 1, \dots$, denote by $x_t$ the search point $x$ at the beginning of the $(t+1)$st iteration, except when the algorithm previously had generated the optimal solution, then let $x_t$ be the optimal solution. With probability $1 - o(1)$, there is a $t_0$ such that $n/6 \le d(x_{t_0}) \le n/3$. We show that in this case, we have an expected optimization time as claimed, which implies that also the unconditioned expectation is of the same order of magnitude. 
  
  For $t = 0, 1, \dots$ define $X_t = \max\{d(x_{t_0+t}),1\}$. Observe that $X_{t+1} \le X_t$ for all $t \ge 0$. Let $s_{\min} := d_0$. Then we have shown above that if $X_t = s > s_{\min}$, then $E[X_t - X_{t+1}] \le K \lambda s / n$ for some absolute constant $K$. Hence the first condition of the drift theorem is fulfilled with $\delta = K \lambda / n$. From Proposition~\ref{prop:lbsmallsteps} we know that $\Pr[X_{t+1} \le s/2] \le \lambda^2 \exp(-\Omega(s)) = \exp(-\Omega(s))$. Hence for $n$ (and thus also $s$) sufficiently large, also the second condition of the drift theorem is satisfied (with $\beta = 1/2$). We may thus apply the theorem and derive that the first $t$ such that $X_t \le s_{\min}$ satisfies $E[t] = \Omega(\frac{\ln(X_0 / s_{\min})}{\delta}) = \Omega(\frac{n \log(k / \log\log n)}{\lambda}) = \Omega(\frac{n \log(k)}{\lambda})$. Note that this, naturally, is a lower bound on $E[T]$. Consequently, $E[F] = \Omega(\lambda E[T]) = \Omega(n \log k)=\omega(F^*)$.
\end{proof}

\section{Conclusion}

We proved that no parameter combination for the \ga can lead to an asymptotically better runtime on the \onemax test function class than the one suggested in~\cite{DoerrDE13}, where this algorithm was first proposed. We also proved that if some offspring population size $\lambda$, some mutation probability $p = k/n$, and some crossover bias $c = r/k$ leads to the asymptotically best runtime, then $\lambda = \Theta(\lambda^*) = \Theta(\sqrt{\log(n) \log\log(n) / \log\log\log(n)})$, $k = \Omega(\lambda^*) \cap \exp(\omega(\sqrt{\log(n) \log\log\log(n) / \log\log(n)}\,))$, and $r = \Theta(1)$. 

A closer inspection of the proofs allows (in a semi-rigorous manner) to extract some hints on the parameter choice also for optimization problems beyond the \onemax test function class. The most clear one is that {$r = \Theta(1)$}, that is, $pc = \Theta(1/n)$, seems a good choice. It was argued intuitively in~\cite{DoerrDE13} that this is a good choice, because it results in that $\cross_c(x,\mut_p(x))$ has the same distribution as applying standard bit mutation to $x$ with the standard choice of $1/n$ for the mutation probability. This intuitive argument is somewhat imprecise due to the fact that one iteration of the \ga contains two selection phases, so neither the winner of the crossover phase has a standard bit mutation distribution (with rate $p$), nor the winner of the crossover phase has a bits taken independently from the mutation winner with probability $c$. Nevertheless, as the proofs of Lemma~\ref{lem:lblargek} and~\ref{lem:lbsmallk} (for a large range of parameter settings) show, in many situations a super-constant value for $r$ leads with high probability to the event that the crossover offspring takes much more ``bad'' bits from the mutation winner than it takes good bits. Conversely, an $r$-value of $o(1)$ together with not too small value for $k$ leads to probability of $\Theta(r)$ for the crossover offspring being equal to the parent $x$, making it useless. 

For the choice of $\lambda$, as with all population based algorithms, it is obvious that larger $\lambda$-value can only be beneficial if the positive effects of the large population outnumber the higher cost for a single iteration. From Lemma~\ref{lem:lbsmallk}, we see that, again for broad ranges of the other parameters, we pay for a too small $\lambda$ when making progress is difficult. A small value of $\lambda$  decreases both the chance to find some good bits in the mutation phase and the chance that good bits are copied into a crossover offspring. This quadratic price for a small $\lambda$ is worth the multiplicative increase of the effort of one iteration. A similar lesson could be deduced from the fitness dependent or the self-adjusting choice of $\lambda$ in~\cite{DoerrDE15,DoerrD15self}, which both again suggest a larger value for $\lambda$ when being closer to the optimum, which in the \onemax landscape means that it is harder to find an improvement.

For the choice of the mutation probability $p = k/n$, the proof of Lemma~\ref{lem:kvalue} shows that a large $k$ can lead to the effect that all mutation offspring look similar. In this case, the mutation phase does not gain from the large $k$-value, whereas in the crossover phase the crossover bias of $c = r/k$ makes it difficult to copy good bits into the final solution. 

We hope that these insights make it easier to use the \ga, which both in theoretical and empirical investigations showed a promising performance. We are also optimistic that the proof ideas developed in this work make future analyses of more-dimensional parameter spaces easier.

}


}
\end{document}